\documentclass[letterpaper, 10 pt, conference]{ieeeconf}  

\IEEEoverridecommandlockouts                              

\usepackage{graphics} 
\usepackage{epsfig} 
\usepackage{mathptmx} 
\usepackage{times} 
\usepackage{amsmath} 
\usepackage{amssymb}  
\usepackage{algorithm}
\usepackage{algpseudocode}
\usepackage{xcolor}

\DeclareMathOperator*{\argmax}{arg\,max}
\DeclareMathAlphabet{\mathcal}{OMS}{cmsy}{m}{n}
\algnewcommand\algorithmicinput{\textbf{Input:}}
\algnewcommand\Input{\item[\algorithmicinput]}

\newtheorem{theorem}{Theorem}
\newtheorem{lemma}{Lemma}
\newtheorem{remark}{Remark}

\title{\LARGE \bf
Cooperative Multi-Agent Graph Bandits: UCB Algorithm and Regret Analysis
}

\author{Phevos Paschalidis~~~~Runyu Zhang~~~~Na Li
 \thanks{P. Paschalidis, R. Zhang, and N. Li are affiliated with Harvard J. Paulson School of Engineering and Applied Sciences. {\small ppaschalidis@college.harvard.edu, runyuzhang@fas.harvard.edu, nali@seas.harvard.edu.}}\thanks{This work was funded by NSF AI institute: 2112085, NSF CNS: 2003111, NSF ECCS: 2328241, and Harvard College Research Program.}
 }

\begin{document}

\maketitle
\thispagestyle{empty}
\pagestyle{empty}

\begin{abstract}
    In this paper, we formulate the multi-agent graph bandit problem as a multi-agent extension of the graph bandit problem introduced in \cite{Zhang2023}. In our formulation, $N$ cooperative agents travel on a connected graph $G$ with $K$ nodes. Upon arrival at each node, agents observe a random reward drawn from a node-dependent probability distribution. The reward of the system is modeled as a weighted sum of the rewards the agents observe, where the weights capture some transformation of the reward associated with multiple agents sampling the same node at the same time. We propose an Upper Confidence Bound (UCB)-based learning algorithm, \texttt{Multi-G-UCB}, and prove that its expected regret over $T$ steps is bounded by $O(\gamma N\log(T)[\sqrt{KT} + DK])$, where $D$ is the diameter of graph $G$ and $\gamma$ a boundedness parameter associated with the weight functions. Lastly, we numerically test our algorithm by comparing it to alternative methods. 
\end{abstract}
\section{Introduction}

The Multi-Armed Bandit (MAB) problem is a fundamental problem in the study of decision-making under uncertainty \cite{lattimore2020bandit, Slivkins_2022, Sutton_Barto_2018, Bubeck2012}. In its simplest formulation, the MAB consists of $K$ independent arms each with an associated reward probability distribution and a single agent that plays the MAB instance for $T$ turns. On each turn, the agent selects one of the arms and receives a reward drawn from the arm's probability distribution. The agent's goal is to maximize the expected reward received over the course of its $T$ turns. In order to do so, the agent must balance \textit{exploring} arms it knows little about and \textit{exploiting} arms already demonstrated to give high rewards. 

One important limitation of the MAB framework is the assumption that at each time step the agent has access to all arms regardless of its previous action. Earlier work introduces a graph bandit problem where the agent traverses an undirected, connected graph, receiving random rewards upon arrival at each node \cite{Zhang2023}. The graph's nodes thus represent the $K$ arms and the local connectivity between nodes imposes restrictions on which arms can be played sequentially. 
This framework has applications in robotics: consider a street-cleaning robot exploring an unknown physical environment or a mobile sensor covering some space to find the location that receives the strongest signals.
However, \cite{Zhang2023} only examines the case of a single agent, whereas in most robotic applications it is useful to consider multiple agents exploring and exploiting the unknown graph together. 

There are many works extending the standard MAB problem to a multi-agent setting. In most studies, agents observe an identical environment where their actions are independent of one another and there exist restrictions on agent communication 
\cite{Zhu2020, Chakraborty2017, Martinez2019, Wang2019, Landgren2019, Agarwal2021, Sankararaman2019, Chawla2020}. Other studies model the multi-agent extension such that two agents selecting the same action ``collide'' and observe no reward \cite{Kalathil2012, Liu2010, Wang2020, boursier2019sic}. Another line of related works focuses on the combinatorial bandit setting, where a single centralized decision maker chooses a ``super-arm'' from some feasibility set $\mathcal{S} \subseteq 2^{[K]}$ and receives a function of the random samples drawn from the arms that comprise the super-arm as a reward \cite{Chen2013, chen2016combinatorial, WangChen2018, Gai_2012, Kveton2015}. Intuitively, we could also consider each individual arm in the super-set as being chosen by a different agent, thus defining the combinatorial bandit framework as a type of multi-agent MAB problem. Importantly, however, none of the multi-agent or combinatorial bandit works consider a graph-like restriction on agent transitions.


\noindent \textbf{Our Contributions.} Motivated by the discussion above, we formulate the multi-agent graph bandit problem as a viable multi-agent extension of the graph bandit problem. Specifically, we consider $N$ cooperative agents co-existing on the same undirected, connected graph $G$ with the objective of optimizing a total reward that is a weighted sum of the rewards of arms that agents select, where the weight associated with each arm depends on the number of agents selecting it. 
We then propose a learning algorithm, \texttt{Multi-G-UCB}, which uses an Upper Confidence Bound (UCB)-based approach to managing the exploration/exploitation trade-off. In particular, the algorithm maintains a confidence radius for each arm that shrinks proportionally to the number of times the arm has been sampled and optimistically assumes that the mean reward for each arm is the maximum value within this radius.
We show that \texttt{Multi-G-UCB} achieves a regret upper bounded by $O\left(\gamma N\log(T)\left[\sqrt{TK} + DK\right]\right)$, where $D$ represents the diameter of the graph $G$ and $\gamma$ is a boundedness coefficient for the weight functions. We also demonstrate via simulations that our proposed algorithm works well empirically, achieving lower regret than several important benchmarks including a multi-agent version of G-UCB where each agent runs the G-UCB algorithm in isolation.

The remainder of the paper is organized as follows. In Section \ref{sec:formulation}, we formulate the multi-agent graph bandit problem and provide examples of potential applications. Section \ref{sec:algorithm} introduces the \texttt{Multi-G-UCB} algorithm, \ref{sec:reganalysis} provides the theoretical analysis of \texttt{Multi-G-UCB}, and \ref{sec:numerical} details the experimental simulation result.

\section{Problem Formulation} \label{sec:formulation}
\subsection{Multi-Agent Graph Bandit}
We define the multi-agent graph bandit problem as follows. Consider an undirected, connected graph with self-loops $G = ([K], \mathcal{E})$ and diameter $D$, where the diameter of a graph refers to 
the maximum distance between any two of its nodes.  For
each node $k \in \{1, \dots, K\} = [K]$, there exists an associated probability distribution for its reward,
$P_k(\cdot)$, with support $[0,1]$ and mean $\mu_k$; the nodes of the graph thus
represent the $K$ arms of the bandit problem. Now, let there exist $N$ individual
agents, each having knowledge of the graph $G$ and able to communicate with the
other agents. Let $T$ be an integer representing the
number of steps in the multi-armed bandit problem. We initialize each agent $i$ to
start on some initial node $k_{i,0}$. Then, at each time step $t \leq T$, each
agent $i$ chooses an arm $k_{i,t}$ such that $(k_{i,t-1}, k_{i,t}) \in
\mathcal{E}$. That is, the agent traverses an edge in the graph to a neighboring
node/arm. For each time step $t$, we define $c_{k,t}$ as a count of the number of agents sampling node $k$ at time $t$, and further define $C_t = (c_{1,t}, \dots, c_{K,t})$. Then, for each $k$ such that $c_{k,t} > 0$, a reward $X_{k,t}$ is drawn independently from $P_{k}(\cdot)$. The algorithm observes a system-wide reward of
\begin{equation*}
  R_t := \sum_{k \in [K]}f_k(c_{k,t})X_{k,t},
\end{equation*}
where for each arm $k$, $f_k$ is defined such that $f_k(1) = 1$ and $f_k(c) \leq \gamma \cdot c$ for some constant $\gamma$ independent of $c$ and $k$. The $f_k$'s thus represent some general transformation of the reward for multiple selections of the same arm at the same time. Note that $f_k$ is not necessarily equivalent to $f_{k'}$
for $k \neq k'$. At the end of each time step $t$, $X_{k,t}$ is revealed for all arms $k$ with $c_{k,t} > 0$. 
The goal of the $N$ agents is thus to travel the graph so as to maximize the net expected reward received by the system of agents over the course of the $T$ time steps. 

Equivalently, we can define the goal as a minimization of expected regret. In order to do so, we first need to define the optimal allocation of agents over the graph. Importantly, since the rewards the system observes at each node are dependent only on the number of agents sampling the node and not the location of specific agents, it suffices to specify the optimal allocation only in terms of the count of agents at each node.
We can thus define the set of possible agent allocations as $\mathcal{C} = \{C=(c_1, \dots, c_K):~\sum_{k \in [K]} c_k = N\}$ and the optimal allocation
\begin{equation*}
    C^\ast := \argmax_{C \in \mathcal{C}} \sum_{k\in[K]} f_k(c_k)\mu_k
\end{equation*}
as the count vector that maximizes the expected reward of the system. We can
then define the regret of the system at time $t$ as
\begin{equation*}
\mathcal{R}_t := \sum_{k \in [K]} f_k(c^\ast_k)\mu_k - R_t.
\end{equation*}

\subsection{Examples}

In this section, we provide two motivating examples that fit our formulation. 

\subsubsection{Drone-Enabled Internet Access}
Consider $N$ cooperative drones deployed over a network of $K$ rural communities to provide internet access. Each drone can serve only one community per time period and, between time periods, only has time to move from its current location to an adjacent community. At the end of each time period, the fleet of drones receives a reward proportional to the communication traffic it serviced. Importantly, for any time period, the demand of each community is sampled independently from a distribution associated with that location; the demand of a community at any time step is stochastic, but different communities have different average internet needs. 
If multiple agents are positioned over the same community, the reward they observe is modified. It may be the case that the marginal benefit decreases with each additional agent, in which case the $f_k$'s would be concave. Our formulation also extends to more complicated interactions, though, for example with agents that can amplify each other's effect if positioned over the same community. The goal of the drone fleet is to maximize the total reward before the robots' batteries run out and they are recalled. Note that this problem is also related to the ``sensor coverage'' formulations of \cite{cortes2004coverage, ramaswamy2016sensor, sun2017submodularity, prajapat2022near, ramaswamy2021multiagent}.

\subsubsection{Factory Production} \label{sec:fac_production}
Consider a factory with $N$ production lines each of which can manufacture a total of $K$ different products, and assume that there are restrictions regarding which products can be manufactured sequentially (depending on the raw materials required, for example). Then, we can also model this problem as an instance of the multi-agent graph bandit problem where each production line is an agent and the actions they take are the products they choose to manufacture at each time step. Each product $k$ is associated with some reward dependent on its stochastic demand market as represented by the reward probability distribution $P_k(\cdot)$. The more production lines that choose to manufacture product $k$, the greater the supply which negatively affects the price of the commodity; thus we have diminishing marginal return for additional lines producing $k$ which can be modelled with concave choice of $f_k(\cdot)$ (though, again, the formulation also extends to more general $f_k$'s).

A related problem in which transitioning between products is subjected not only to binary yes/no constraint but also associated with some cost can be represented by a framework extremely similar to ours with the addition of non-uniform edge weights to $G$. The algorithm we introduce, \texttt{Multi-G-UCB}, can be modified to address this formulation with only a slight modification to the offline planning component and the regret analysis (see Remarks \ref{rem:off} and \ref{rem:reg}).

\section{Multi-Agent Graph Bandit Learning} \label{sec:algorithm}

In this section, we present a learning algorithm for the multi-agent graph bandit problem. In Section \ref{subsec:algoverv} we outline the algorithm and provide intuition for its structure; in Section \ref{subsec:sp} we explore in depth the offline planning algorithm we use as a subroutine; and in Section \ref{subsec:init} we discuss the algorithm's initialization phase.

\subsection{Algorithm} \label{subsec:algoverv}
The formulation given in Section \ref{sec:formulation} introduces a few key complications with respect to the current literature. If we merely apply MAB algorithms to the graph setting, then at any time step the intended arm for an agent can be far away on the action graph $G$, requiring multiple time steps of sub-optimal actions as the agent traverses the path between the initial and desired arm. Furthermore, in comparison to the single-agent graph bandit setting, the existence of multiple agents and the non-linear weight functions necessitates communication and sophisticated planning since the reward of each agent is dependent on the actions of other agents.


With these in mind, we introduce the \texttt{Multi-G-UCB} algorithm which attempts to minimize this transition cost by dividing the time horizon $T$ into $E$ episodes and emphasizes agent communication through communal UCB values. In each episode, the algorithm performs four major steps: 
(i) integrating agent information to compute the UCB values for each node, (ii) calculating the desired allocation using the UCB values, (iii) transitioning to the destination allocation using an offline planning algorithm, \texttt{SPMatching}, and (iv) collecting new samples at the destination nodes until samples double.

Specifically, \texttt{Multi-G-UCB} calculates for
each arm the mean observed reward $\{\hat{\mu}_{k,t_e}\}_{a \in [K]}$ as well as a
confidence radius and upper confidence bound
\begin{equation}
  b_{k,t_e} := \sqrt{\frac{2\log(t_e)}{n_{k,t_e}}} \quad\mathrm{and}\quad
  U_{k,t_e}  := \hat{\mu}_{k,t_e} + b_{k,t_e},
\end{equation}
where we denote by $t_e$ the time at which episode $e$ starts and by $n_{k, t_e}$ the
number of times arm $k$ has been sampled in the first $t_e$ time steps. Importantly,
in calculating the mean reward, the algorithm uses the $X_{k,t_e}$ drawn from the
arm's reward distribution rather than the observed reward weighted by
$f_k(\cdot)$ and utilizes the collective experience of all agents. Similarly, we update the sample counts as
\begin{equation}
n_{k,t+1} := \begin{cases}
   n_{k,t}, & \text{if $c_{k,t} = 0$,} \\
   n_{k,t} + 1, & \text{otherwise,}
\end{cases} 
\end{equation}
so that they reflect the number of samples drawn from the reward distribution. 

In order to calculate the optimal allocation of agents, \texttt{Multi-G-UCB} optimistically
assumes that the true mean reward of each arm is equal to its upper confidence bound. Thus, we define the estimated optimal count vector of episode $e$ as
\begin{equation} \label{eq:maxcount}
    \hat{C}_e := \argmax_{C \in \mathcal{C}} \sum_{k\in[K]} f_k(c_k)U_{k,t_e}.
\end{equation}
We also define $k_{min}$ as the arm with the fewest observed samples among all arms $k$ such that $\hat{c}_{k,e} > 0$ and $n_{min} = n_{k_{min}, t_e}$ as its current number of samples. These quantities will be important in defining the length of our episode.

Following the calculation of $\hat{C}_e$, the algorithm relies on an offline planning
sub-routine, \texttt{SPMatching}, defined in Section \ref{subsec:sp} to transition the
system of agents from the current state to the desired one. Following this transition
phase, the algorithm exploits the desired state by repeatedly sampling the (potentially
suboptimal) allocation until the number of samples of the arm with the fewest prior samples, i.e. $k_{min}$, doubles; that is until $n_{k_{min}, t} = 2n_{min}$. The doubling scheme is a well-known technique in reinforcement learning, e.g. \cite{jaksch10a}. We would like to remark that our specific choice of doubling scheme, doubling the least-sampled arm, is crucial in the regret analysis. Other schemes such as doubling the most-sampled arm would lengthen the exploitation phase upsetting the exploration/exploitation balance of the algorithm. This conclusion is further supported by the numerical results in Section \ref{sec:numerical} when we compare our proposed doubling scheme with two others.
The pseudocode for \texttt{Multi-G-UCB} is given by Algorithm
\ref{alg:GcompUCB}.

\begin{algorithm}[hbt]
\caption{\texttt{Multi-G-UCB:}} \label{alg:GcompUCB}
\begin{algorithmic}[1]
\Input The initial nodes for each agent, $\{k_{i,0}\}_{i \in [N]}$. The offline planning algorithm \texttt{SPMatching} (see Section \ref{subsec:sp}) that computes the optimal policy given the graph $G$ and a set of computed UCB values, $\{U_{k,t_e}\}_{k \in [K]}$, for each node.
\State $e \gets 0$
\State Run the initialization algorithm to visit each vertex in $G$. See Section \ref{subsec:init}.
\While{coordinator has not received a stopping signal}
    \State $e \gets e+1$
    \State Calculate the UCB values $\{U_{k,t_e}\}_{k \in [K]}$.
    \State Solve the optimization problem in (\ref{eq:maxcount}) for the desired allocation $\hat{C}_{e}$.
    \State Denote by $k_{min}$ the arm with fewest observed samples of all arms $k$ such that $\hat{c}_{k,e} > 0$ and denote by $n_{min}$ its current number of samples ($n_{min} = n_{k_{min}, t_e}$).
    \State Follow the offline planning policy \texttt{SPMatching} until agents are distributed according to $\hat{C}_{e}$.
    \State Have each agent $i$ continue to collect rewards at its current node until $n_{k_{min}, t} = 2n_{min}$. 
\EndWhile
\end{algorithmic}
\end{algorithm}

\subsection{Offline Planning} \label{subsec:sp}

The goal of the offline planning subroutine is to transition from the current state
to one in which the agents are distributed according to $\hat{C}_{e}$
while incurring the least regret possible. In the single-agent setting,
the ``allocation'' of agents will always consist of just a single
destination node, and thus this transition-of-least-regret is equivalent to a
shortest path problem where the weight of each edge $(k', k) \in \mathcal{E}$ is the
estimated regret of arm $k$. More formally, these paths---which we call
\textit{regret shortest paths}---are shortest paths on a graph $G' = ([K],
\mathcal{E}, W)$ where the weights in $W$ are individually defined as
\begin{equation} \label{eq:weights}
    w(k', k) = \max_{v \in [K]} (U_{v,t_e} - U_{k,t_e}).
\end{equation}
Importantly, we only allow paths of maximum length $D$, noting that each pair of vertices is assured to have such a path by definition of our graph diameter.

In our multi-agent setting, the solution is more complicated. We propose a
polynomial-time pseudo-solution, which we call \texttt{SPMatching}. In particular, we define $\mathcal{S}_{e} = ([K], \hat{C}_{e})$ as a multiset of arms such that the multiplicity of each arm $k \in [K]$ is equal to $\hat{c}_{k,e}$. We then create a complete bipartite graph $G_B = ([N],
\mathcal{S}_{e}, \mathcal{E}_b)$
where the weight of each $(i, k) \in \mathcal{E}_b$ is the length of the regret 
shortest path between the current location of agent $i$ and arm $k$.
We can then calculate the minimum
weighted perfect matching of $G_B$ to assign agents a corresponding arm in
$\mathcal{S}_{e}$ and instruct each agent to follow the shortest path towards
that arm \cite{Zvi1986}. The pseudocode for \texttt{SPMatching} is given in Algorithm
\ref{alg:spmatching}.

We call \texttt{SPMatching} a pseudo-solution because there are no theoretical
guarantees on its minimization of regret during the transition phase. Our restriction on the physical path length of any regret shortest path means that our paths may be suboptimal, but even without this constraint the regret shortest paths do not account for the actions of the other agents. In particular, if the optimal paths of two agents intersect at some node, then the reward observed at that node is potentially less than the reward estimated by the shortest path calculations. This increases regret in a manner unanticipated by the offline planning algorithm. As we see in Section \ref{sec:reganalysis}, the sub-optimality of our transition phase does not adversely affect our regret since its worst-case scenario is equivalent to the worst-case of an optimal solution.

\begin{algorithm}[hbt]
\caption{\texttt{SPMatching}} \label{alg:spmatching}
\begin{algorithmic}[1]
\Input The current time $t$, the current position of the agents, $\{k_{i,t}\}_{i \in [N]}$, and a desired allocation of agents, $\hat{C}_{e}$.
\State Calculate the multiset of destination arms $\mathcal{S}_{e} = ([K], \hat{C}_{e})$.
\State For each agent $m$ calculate the regret shortest path between its current location and each of the destination arms.
\State Initialize a complete bipartite graph $G_B = ([N], \mathcal{S}_{e}, \mathcal{E}_B)$ where the weight of each edge $(i, k) \in \mathcal{E}_B$ is defined as the length of the regret shortest path between the current location of $i$ and the destination arm $k$.
\State Calculate the minimum weight perfect matching of $G_B$.
\State Have each agent follow the regret shortest path to their assigned destination node.
\end{algorithmic}
\end{algorithm} 

\begin{remark}[Algorithmic Extension to Weighted Graphs] \label{rem:off}
    In order to extend \texttt{Multi-G-UCB} to the case of weighted arm graphs $G$ as discussed in Section \ref{sec:fac_production} we would need only to change the definition of the regret shortest paths. In particular, if the cost of transition from arm $k'$ to arm $k$ was $\alpha(k',k)$, then by redefining the weight of edge $(k', k)$ in $G'$ as $w(k', k) = \alpha(k',k) + \max_{v \in [K]} (U_{v,t_e} - U_{k,t_e})$
    we fully account for these transition costs.
\end{remark}

\subsection{Initialization} \label{subsec:init}

In order for the algorithm to begin in earnest, it needs to have observed at least
one reward from each arm so that the upper confidence bounds are
well-defined. Therefore, we begin the algorithm with an initialization episode in which each agent runs an independent graph
exploration algorithm, Depth-First Search (DFS), until all vertices have been visited at
least once \cite{cormen2022introduction}. Despite the potential redundancy, we will see in Section
\ref{sec:reganalysis} that the initialization phase does not contribute to the
asymptotic regret of the system.

\begin{remark}[Algorithmic Extension to Directed Graphs]
    We remark that the initialization phase is the only time that our algorithm uses the assumption that the original graph $G$ is undirected. This assumption is necessary for the DFS algorithm, which relies on backtracking. Therefore, given a different initialization technique, \texttt{Multi-G-UCB} could also be applied to a multi-agent graph problem with a directed $G$ which would reflect asymmetric transition constraints.
\end{remark}

\section{Main Results} \label{sec:reganalysis}

Having defined our algorithm, we now turn to a theoretical analysis of its regret. We
summarize our findings in the Theorem \ref{th:reg}.
\begin{theorem}[Regret of \texttt{Multi-G-UCB}]
    Let $T \geq 1$ be any positive integer. Given an instance of a multi-agent graph bandit problem with $N$ agents, $K$ arms each with weight function $f_k(\cdot)$ bounded such that $f_k(c) \leq \gamma c$, and a graph diameter of $D$, the expected system-wide regret of \texttt{Multi-G-UCB} after taking a total of $T$
    steps (including initialization) is bounded by
    \begin{equation}
      \mathbb{E}\left[\sum_{t=1}^T \mathcal{R}_{t}\right] \leq O\left(\gamma N\log(T)\left[\sqrt{KT} + DK\right]\right).
    \end{equation}
    \label{th:reg}
\end{theorem}

Note that our result matches the regret bound for single-agent graph bandit \cite{Zhang2023} when $N=1$  (with a subtle difference on the dependency on the $\log T$ factor). It is also worth noting that the regret in our multi-agent graph bandit formulation exhibits a linear growth with respect to the number of agents $N$, which stands in contrast to the results of other algorithms (e.g. \cite{Kveton2015}) that solve multi-agent MAB with regret $O(\sqrt{KNT\log T})$. This discrepancy arises primarily from our consideration of the weight functions $f_k$ for the total objective which complicates the analysis, as opposed to the unweighted setting (i.e. $f_k(\cdot) = 1$) considered in \cite{Kveton2015}. Indeed, in the case of the unweighted setting, it can be shown that our algorithm also achieves regret proportional to $\sqrt{N}$. It is an interesting open question whether we can further tighten the bound on the dependency on $N$ for general $f_k(\cdot)$.

\hfill 

\textit{Proof Ideas.}
The detailed proof can be found in the full version of the paper, posted to arXiv \cite{paschalidis2024cooperative}. Here we outline the main ideas. In order to prove Theorem \ref{th:reg}, we analyze the regret of the system by episode, noting that since our algorithm doubles the samples of at least one arm each episode we can bound the number of episodes as $K\log(T)$. For each episode, we define the ``good'' event as
\begin{equation}
    \mathcal{G}_e = \{\forall k \in [K], \mu_k \in [\hat{\mu}_{k,t_e} - b_{k,t_e}, \hat{\mu}_{k, t_e} + b_{k, t_e}]\}.
\end{equation}
This is the event where the true mean of each arm is within the arm's confidence radius of its estimated mean. The ``bad'' event is thus the probability that at least one of our confidence bounds fails. Following this decomposition, our regret proof consists of four steps:
(1) bounding the regret contribution of the bad episodes, (2) bounding the transition phase of the good episodes, (3) bounding the exploitation phase of the good episodes, and (4) bounding the regret of initialization.

\textbf{Step 1: Regret Contribution of Bad Events.} We use Hoeffding's inequality to bound the regret of the bad episodes by showing that for each episode $e$ the probability that any single confidence radius fails is $2t_e^{-3}$. We then note that for any single time step we can incur regret no more than $\gamma N$ since our rewards are bounded in $[0, 1]$ and the weight functions linear in the counts, so using a union bound over the arms and summing over all episodes gives a final regret contribution of 
\begin{equation}
    O(\gamma N(\log(K) + \log\log(T)).
\end{equation}

\textbf{Step 2: Regret of Transition Phase.} As discussed, we divide the good episodes into two phases. The first phase, the transition phase, is when the agents move along the graph to their assigned destination nodes. Since this takes at most $D$ steps---recall our restriction on the regret shortest paths---and at each time step we incur a regret of no more than $\gamma N$, we can bound the total transition cost after summing over all episodes by 
\begin{equation}
    O(\gamma N\log(T)DK).
\end{equation} 
Note that since we assume the worst possible matching, the fact that our offline planning sub-routine \texttt{SPMatching} is only a pseudo-solution does not affect our theoretical regret analysis. 

\textbf{Step 3: Regret of Exploitation Phase.} The second phase of the good episodes is the exploitation phase, when the system repeatedly samples from the estimated optimal arm allocation represented by $\hat{C}_{e}$. In analyzing this phase, we make use of our definitions for $b_{k, t_e}$ and $U_{k, t_e}$ to simplify the regret at each time step as
\begin{multline}
    \sum_{k \in [K]} 
        \left[f_k(c_k^\ast) - f_k(\hat{c}_{k, e})\right]\mu_k 
    \\ \leq 
    2\sum_{k \in [K]} 
        f_k(\hat{c}_{k,e}) b_{k, t_e}
    + \sum_{k \in [K]}
        \left[f_k(c_{k}^\ast) - f_k(\hat{c}_{k, e})\right]U_{k, t_e}
    \\ 
    \leq 2\sum_{k \in [K]} 
        f_k(\hat{c}_{k, e}) b_{k, t_e},
\end{multline}
where the last inequality follows from definition of $\hat{C}_{e}$ as the count vector that maximizes $\sum_{k \in K} f_k(c_k)U_{k, t_e}$. Then, we simplify this sum by extracting a $\sqrt{2\log(t_e)} \leq \sqrt{2\log(T)}$ from each confidence radius and showing
\begin{equation} \label{eq:conf}
    \sum_{k \in [K]} \frac{f_k(\hat{c}_{k,e})}{\sqrt{n_{k,t_e}}}
    \leq \sum_{k \in [K]} \frac{\gamma \hat{c}_{k,e}}{\sqrt{n_{k,t_e}}} 
    \leq \frac{\gamma N}{\sqrt{n_{min}}}
\end{equation}
by noting that for each destination arm $k$, $n_{k, t_e} \geq n_{min}$ and that $\sum_{k \in [K]} \hat{c}_{k,e} \leq N$. A final bound of
\begin{equation}
    O(\gamma N\log(T)\sqrt{NK})
\end{equation}
follows when we sum over all episodes since for each episode $e$ the exploitation phase cannot exceed $n_{min}$ steps. 

\textbf{Step 4: Initialization Cost.} Finally, the initialization cost contributes just $O(\gamma NK)$ to the regret since each of $N$ agents runs DFS until the $K$ vertices have been visited.

\hfill

Combining these four steps gives the desired regret upper bound. Note that Steps 1 and 4 are asymptotically dominated by Steps 2 and 3 so that our final regret bound is the sum of the transition phase and exploitation phase regret.

\begin{remark}[Regret Consideration of Weighted Graphs] \label{rem:reg}
    The consequences of a weighted $G$ to the theoretical analysis would manifest in Steps 1, 3, and 4. In particular, the assumption that each time step incurs a worst-case regret of $N$ would need to be updated to reflect the transition costs potentially incurred by the agents. So long as these costs are bounded by some constant $B$, we conclude that at each time step the regret incurred is no more than $(1 + B)N$ which does not effect our asymptotic regret bound.
\end{remark}

\section{Numerical Simulations} \label{sec:numerical}

To test the performance of our model experimentally, we construct a synthetic multi-agent graph bandit problem with $K = 300$ arms and $N = 20$ agents. The graph, which we display in Figure \ref{fig:graph}, is initialized as an Erdos-Renyi graph with probability parameter 0.05. Each agent is assigned a random start vertex. The reward distribution for each arm is distributed as $\mathcal{N}(\mu, \sigma^2)$ with $\mu$ chosen uniformly from the interval $(0.25, 0.75)$ and $\sigma^2 = 0.06$. Motivated by the case of decreasing marginal rewards, for each arm $k \in [K]$, we also define the concave function 
\begin{equation*}
    f_k(x) = \frac{\log_{2+k} (\frac{x}{20} + \frac{1}{2+k}) + 1}{\log_{2+k} (\frac{1}{20} + \frac{1}{2+k}) + 1}.
\end{equation*}
We use a state-of-the-art optimization software, Gurobi, to solve the non-convex optimization problem given by (\ref{eq:maxcount}) \cite{gurobi}.

In addition to \texttt{Multi-G-UCB}, we define three benchmark algorithms. \texttt{Multi-G-UCB-median} is a variation where for each episode we sample the estimated optimal arm allocation until the arm with the median number of prior samples has its sample count doubled (rather than the minimum). \texttt{Multi-G-UCB-max} is defined similarly but uses the arm with the maximum number of prior samples as the baseline for each episode. Lastly, \texttt{Indv-G-UCB} is derived by having each agent individually perform the G-UCB algorithm of \cite{Zhang2023} without communication and coordination. Figure \ref{fig:cumreg} displays the cumulative regret of each algorithm as a function of time averaged over 10 independent trials each with time horizon $T = 1.5 \cdot 10^{5}$.

\begin{figure}[thpb]
    \begin{minipage}{0.2\textwidth}
      \centering
      \includegraphics[width =\linewidth]{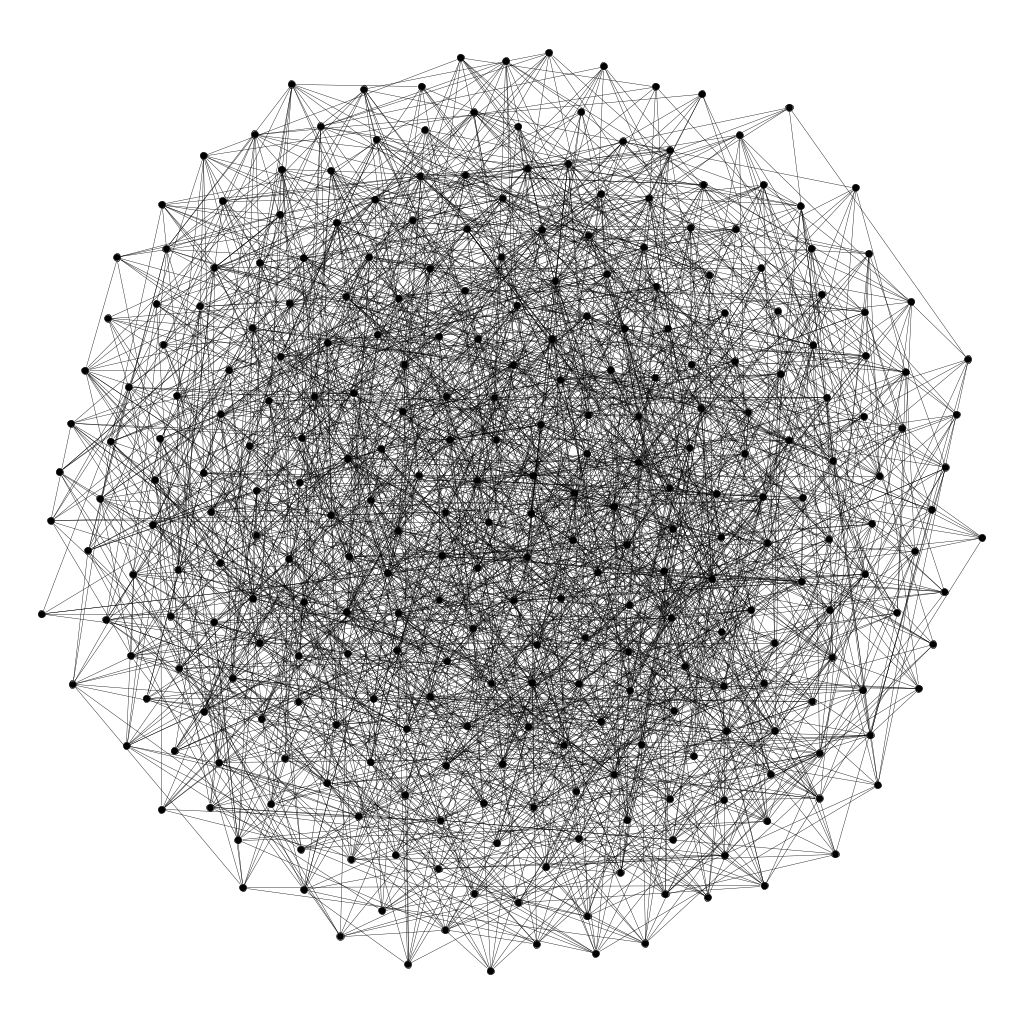}
      \caption{The graph $G$.}
      \label{fig:graph}
    \end{minipage} \hfill
    \begin{minipage}{0.28\textwidth}
      \centering
      \includegraphics[width = \linewidth]{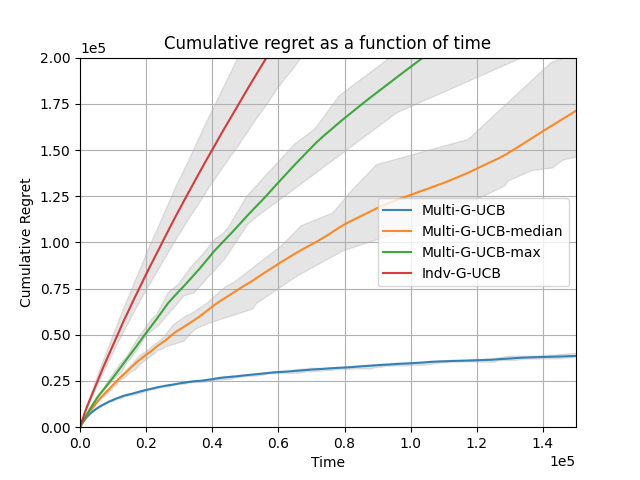}
      \caption{The average cumulative regret of each algorithm across 10 trials.}
      \label{fig:cumreg}
    \end{minipage}
\end{figure}

As expected, all three of the cooperative algorithms considerably outperform \texttt{Indv-G-UCB}. \texttt{Multi-G-UCB} also outperforms its two variations both of whom over-prioritize exploitation and under-prioritize exploration. For \texttt{Multi-G-UCB-max} in particular, once one arm has been shown to be part of the optimal set it becomes increasingly harder to find the other arms since the episode lengths continue to increase even when including relatively unknown arms in the destination allocation.


\section{CONCLUSIONS AND FUTURE WORK}

In this paper, we define the multi-agent graph bandit problem, propose a learning algorithm, \texttt{Multi-G-UCB},
provide a theoretical analysis of its regret and present experimental results that validate its performance. We believe that the formulation multi-agent graph bandit problem opens up many interesting future directions, such as extending our algorithm to accommodate decentralized learning scenarios, generalizing to more complicated case where rewards of arms are correlated across time and/or across arms, as well as performing more comprehensive theoretical studies such as analyzing instance-dependent regret.





\bibliographystyle{bibtex/IEEEtran.bst}
\bibliography{bibtex/IEEEabrv, bibtex/bib}



\section*{APPENDIX} \label{app:proof}
\section*{PROOF OF THEOREM \ref{th:reg}}

Before beginning the proof of Theorem \ref{th:reg}, we define a few relevant quantities. Let $s_{k,e}$ be the number of samples of arm $k$ in episode $e$; that is, $s_{k,e} = n_{k,t_{e+1}} - n_{k, t_{e}}$. Let $H_{e}$ be the length of episode $e$; that is $H_{e} = t_{e+1} - t_{e}$. Furthermore, let $\mathcal{R}_e$ be the total regret of episode $e$ and let $\mathcal{F}_e$ be the smallest $\sigma$-algebra such that $H_0, \dots, H_e$ are measureable.

We also derive a useful result about the number of episodes $E$ of the algorithm.

\hfill

\begin{lemma}[Number of episodes] \label{lem:numepisodes}
    The number of episodes $E$ is logarithmic in $T$. In particular,
    $$E \leq K\log(T).$$
\end{lemma}
\begin{proof}
    Consider any arm $k \in [K]$. Denote by $E_k = \{e_1, e_2, \dots, e_{l(k)}\}$ the ordered set of episodes where $k$ is the destination node with the fewest number of samples and the algorithm is not terminated prematurely. That is, $E_k$ are the episodes whose length is defined such that $s_{k,e} > n_{k,t_{e}}$ and thus we have for each $i \in [2, \dots, l(k)]$
    $$n_{k,t_{e_i+1}} \geq 2n_{k, t_{e_i}} \geq 2n_{k, t_{e_{i-1}+1}}$$
    where the last inequality follows since in all episodes $e \notin E_k$, we still have $s_{k,e} \geq 0$. Note that $n_{k,t_{e_{i-1}+1}}$ is the count (number of samples) of arm $k$ immediately after episode $e_{i-1}$ and that for all episodes $e$ between $e_{i-1}$ and $e_i$, $e \notin E_k$.
    Therefore, 
    \begin{gather*}
        n_{k, t_{e_{i}+1}} \geq 2^{i-1} n_{k, t_{e_{1}+1}} \geq 2^i \\
        \therefore~n_{a, t_{e_{l(k)}+1}} \geq 2^{l(k)} 
    \end{gather*}
    where we used the fact that right after episode $e_1$ we have at least two samples of arm $k$ due to initialization. Every time step, we sample at most $\max(N, K)$ arms so by definition, we have
    \begin{align*}
        TK \geq \sum_{k \in [K]} n_{k, t_{e_{l(k)}+1}} 
        &\geq \sum_{k \in [K]} 2^{l(k)} \\ 
        &\geq K2^{\sum_{k \in [K]} \frac{l(k)}{K}} \geq K2^\frac{E-1}{K}
    \end{align*}
    where the second to last inequality is from Jensen's inequality. Note that we use $E-1$ rather than $E$ since the last episode is potentially cut short so as to not exceed the time limit $T$. Taking logs of both sides gives
    \begin{gather*}
        E-1 \leq K\log(T)/\log(2) \\
        \therefore~ E \leq K\log(T)
    \end{gather*}
    which completes the proof.
\end{proof}

\hfill

Now, we are ready to prove Theorem \ref{th:reg}. As discussed, the initialization phase contributes $O(\gamma NK)$ to the regret. We want to obtain a bound on
\begin{equation*}
    \mathbb{E}\left[\sum_{i=1}^T \mathcal{R}_t\right] = \mathbb{E}\left[\sum_{e=1}^E \mathcal{R}_e\right].
\end{equation*}
Let us now consider some arbitrary episode $e$. We decompose into good and bad events, where we define the ``good'' event as
\begin{equation*}
    \mathcal{G}_e = \{\forall k \in [K], \mu_k \in [\hat{\mu}_{k, t_e} - b_{a, t_e}, \hat{\mu}_{k, t_e} + b_{k, t_e}]\},
\end{equation*}
the event where the true mean of each arm is within the arm's confidence radius of its estimated mean. The ``bad'' event occurs when at least one confidence radius fails. Then, letting $\mathbb{E}_e = \mathbb{E}[\cdot|\mathcal{F}_e]$ and $\mathbb{P}_e = \mathbb{P}[\cdot|\mathcal{F}_e]$ we can write the regret of episode $e$ as
\begin{multline}
    \mathbb{E}\left[\mathcal{R}_e\right] =
    \mathbb{E}
    \biggl[
        \mathbb{E}_e[\mathcal{R}_e]
    \biggr] \\
     =
    \mathbb{E}
    \biggl[
    \mathbb{E}_e[\mathcal{R}_e|\neg\mathcal{G}_e]
    \mathbb{P}_e[\neg\mathcal{G}_e]
    \biggr]
    +
    \mathbb{E}
     \biggl[
     \mathbb{E}_e[\mathcal{R}_e|\mathcal{G}_e]
     \mathbb{P}_e[\mathcal{G}_e]
     \biggr].
    \label{eq:good_bad_decomp}
\end{multline}

\addtolength{\textheight}{-1cm}   

We begin with the first term, the price of failing confidence bounds. First, we note that
\begin{multline}
    \mathbb{E}[\mathcal{R}_e|\neg\mathcal{G}_e, \mathcal{F}_e]  =
    \mathbb{E}\left[\sum_{t = t_e}^{t_{e+1}-1} 
            \sum_{k \in [K]} [f_k(c_k^\ast)- f_k(c_{k, t})]\mu_k 
            \middle| \mathcal{F}_e\right]
            \\ 
    \leq \mathbb{E}\left[H_e \sum_{k \in [K]} f_{k}(c_k^\ast) \middle|\mathcal{F}_e \right]
    \leq \gamma NH_e,
    \label{eq:exp_bad_reg}
\end{multline}
since $\mathbb{E}[H_e|\mathcal{F}_e] = H_e$ and $\sum_{k \in [K]} f_k(c_k^\ast) \leq \gamma \sum_{k \in [K]} c_k^\ast = \gamma N$.
Now, we turn to bounding
\begin{align*}
    \mathbb{P}_e[\neg\mathcal{G}_e] &\leq 
    \sum_{k \in [K]} 
        \mathbb{P}_e\bigl[|\hat{\mu}_{k,t_e} - \mu_{k}| > b_{k, t_e}\bigr] \\
    &\leq \sum_{k\in [K]} \sum_{j=1}^{t_e-1} 
        \mathbb{P}_e\left[\bigl|\hat{\mu}_{k,t_e} - \mu_{k}\bigr| > \sqrt{\frac{2\log(t_e)}{n_{k, t_e}}}\middle| n_{k, t_e} = j \right],
\end{align*}

where in the first inequality we used a union bound over all possible arms, and in the second we used a union bound over all possible values of $n_{k, t_e}$ for every arm. Then, for any specific $k$ and $j$, we have by Hoeffding's inequality
\begin{equation*}
    \mathbb{P}_e\left[\bigl|\hat{\mu}_{k,t_e} - \mu_{k}\bigr| > \sqrt{\frac{2\log(t_e)}{j}}\right] \leq 2~\mathrm{exp}(-4\log t_e) = 2t_e^{-4}.
\end{equation*}
As a result, 
\begin{equation}
    \mathbb{P}_e[\neg\mathcal{G}_e] \leq 2Kt_e^{-3}.
    \label{eq:prob_bad_events}
\end{equation}
To combine (\ref{eq:exp_bad_reg}) and (\ref{eq:prob_bad_events}) we note that $K < t_e$ due to initialization, and that $H_e \leq K + t_e$ since it takes at most $K$ steps to transition to the optimal period after which the algorithm stays stationary for $n_{min} < t_e$ steps. Therefore,
\begin{equation}
    \mathbb{E}
    \biggl[
    \mathbb{E}_e[\mathcal{R}_e|\neg\mathcal{G}_e]
    \mathbb{P}_e[\neg\mathcal{G}_e]
    \biggr] =
    \mathbb{E}
    [2\gamma NH_eKt_e^{-3}] \leq \mathbb{E}[4\gamma Nt_e^{-1}].
    \label{eq:bad_events_contr}
\end{equation}

\hfill

We now offer an analysis of the second term of our good/bad decomposition in (\ref{eq:good_bad_decomp}). We bound $\mathbb{P}_e[\mathcal{G}_e] \leq 1$. Let $D_e$ be the number of transition steps in episode $e$ so that for $t \geq t_e + D_e$ we have $C_t = \hat{C}_{e}$. Note that $D_e \leq D$ by our assumption that each regret shortest path have maximum physical length of $D$. Then,
\begin{multline*}
    \mathbb{E}_e[\mathcal{R}_e|\mathcal{G}_e] =
    \mathbb{E}_e
    \left[
        \sum_{t = t_e}^{t_{e+1}-1} 
            \sum_{k \in [K]} [f_k(c_k^\ast) - f_k(c_{k,t})]\mu_k
    \middle| \mathcal{G}_e
    \right] 
    \\ \leq 
    \gamma DN + 
    \mathbb{E}_e\left[ 
    \sum_{t=t_e+D_e}^{t_{e+1}-1} \sum_{k \in [K]} 
        [f_k(c_k^\ast) - f_k(\hat{c}_{k, e})]\mu_k 
    \middle| 
    \mathcal{G}_e\right] ,
\end{multline*}
since after $D$ steps we will have reached the desired arm allocation represented by $\hat{C}_{e}$. We now note that by definition of $\mathcal{G}_e$, for every $k \in [K]$ we have 
\begin{equation*}
    U_{k, t_e} - 2b_{k, t_e} = \hat{\mu}_{k, t_e} - b_{k, t_e} \leq \mu_k \leq \hat{\mu}_{k, t_e} + b_{k, t_e} = U_{k, t_e}.
\end{equation*} 

Therefore, we can simplify the expectation above as
\begin{multline}
    \mathbb{E}_e\left[ 
    \sum_{t=t_e+D_e}^{t_{e+1}-1} \sum_{k \in [K]} 
        [f_k(c_k^\ast) - f_k(\hat{c}_{k, e})]\mu_k
    \middle| 
    \mathcal{G}_e\right] \\
    \leq 
    2\mathbb{E}_e\left[
    \sum_{t_e + D_e}^{t_{e+1}-1} \sum_{k \in [K]} 
        f_k(\hat{c}_{k, e}) b_{k, t_e} \middle| 
    \mathcal{G}_e\right]
    \\ + \mathbb{E}_e\left[
    \sum_{t_e + D_e}^{t_{e+1}-1} \sum_{k \in [K]}
        [f_k(c_{k}^\ast) - f_k(\hat{c}_{k, e})]U_{k, t_e}
    \middle| 
    \mathcal{G}_e\right] \\ 
    \leq 2\mathbb{E}_e\left[
    \sum_{t_e + D_e}^{t_{e+1}-1} \sum_{k \in [K]} 
        f_k(\hat{c}_{k, e}) b_{k, t_e} \middle| 
    \mathcal{G}_e\right],
    \label{eq:good_events_contr}
\end{multline}
where the last inequality follows from definition of $\hat{C}_{t_e}$ as the count vector that maximizes $\sum_{k \in K} f_k(c_k)U_{k, t_e}$.

\hfill

Combining (\ref{eq:bad_events_contr}) and (\ref{eq:good_events_contr}) and summing over all episodes thus gives an expected regret of
\begin{multline}
    \mathbb{E}[\sum_{t=1}^T \mathcal{R}_t] = 
    2\mathbb{E}\Biggl[\sum_{e=1}^E \left(\mathbb{E}\left[
    \sum_{t_e + D_e}^{t_{e+1}-1} \sum_{k \in [K]} 
        f_k(\hat{c}_{k, e}) b_{k, t_e} 
    \middle| 
    \mathcal{G}_e\right]\right)\Biggr]
    \\ + \mathbb{E}\left[\sum_{e=1}^E \gamma DN\right] 
    + \mathbb{E}\left[\sum_{e=1}^E 4\gamma Nt_e^{-1}\right].
    \label{eq:reg_in_3_terms}
\end{multline}

We bound each of the three terms in (\ref{eq:reg_in_3_terms}). We begin with the first term. We invoke Lemma 4 in \cite{Zhang2023} which states that for any random variable $x$ satisfying $x \geq 0$ almost surely, the equation
\begin{equation*}
    \mathbb{E}[x|\mathcal{G}_e] \leq 2\mathbb{E}[x]
\end{equation*}
holds for all episodes $e$. We can therefore get rid of the conditioning on $\mathcal{G}_e$. 
\begin{multline*}
     2\mathbb{E}\Biggl[\sum_{e=1}^E \left(\mathbb{E}\left[
    \sum_{t_e + D_e}^{t_{e+1}-1} \sum_{k \in [K]} 
        f_a(\hat{c}_{k, e}) b_{k, t_e} 
    \middle| 
    \mathcal{G}_e\right]\right)\Biggr] 
    \\ =  
    4\mathbb{E}
    \left[
    \sum_{e=1}^E \sum_{t=t_e+D_e}^{t_{e+1}-1} \sum_{k \in [K]}
        f_k(\hat{c}_{k, e}) b_{k, t_e}
    \right] 
    \\ \leq 
    4\sqrt{2\log(T)}\mathbb{E}\left[\sum_{e=1}^E 
        \sum_{t = t_e+D_e}^{t_{e+1}-1}
            \sum_{k \in [K]} \frac{f_k(\hat{c}_{k,e})}{\sqrt{n_{k,t_e}}}
    \right].
\end{multline*}
Now, recall that $n_{min}$ is the number of samples of $k_{min}$ at the start of episode $e$ where $k_{min}$ is the arm with the fewest observed samples such that $\hat{c}_{k, e} > 0$. Then, since we know $\sum_{k \in [K]} \hat{c}_{k,e} = N$ we have 
\begin{equation}
    \sum_{k \in [K]} \frac{f_k(\hat{c}_{k,e})}{\sqrt{n_{k,t_e}}}
    \leq \sum_{k \in [K]} \frac{\gamma \hat{c}_{k,e}}{\sqrt{n_{k,t_e}}} 
    \leq \frac{\gamma N}{\sqrt{n_{min}}}.
\end{equation}
Note now that by definition of our episode length we have at most $n_{min}$ time steps at the destination node since we end the episode exactly when $k_{min}$ has its number of samples doubled. Therefore,
\begin{multline*}
    4\sqrt{2\log(T)}\mathbb{E}\left[\sum_{e=1}^E 
    \sum_{t = t_e+D_e}^{t_{e+1}-1}
        \sum_{k \in [K]} \frac{f_k(\hat{c}_{k,e})}{\sqrt{n_{k,t_e}}}
    \right] 
    \\
    \leq 4\gamma \sqrt{2\log(T)}\mathbb{E}\left[\sum_{e=1}^E 
    n_{min} \left(\frac{N}{\sqrt{n_{min}}}\right)
    \right] 
    \\ \leq 
    4\gamma N\sqrt{2\log(T)}\mathbb{E}\left[\sum_{e=1}^E \sqrt{n_{min}}\right].
\end{multline*}

Considering now the expression still in the expectation and applying Jensen's law we get
$$E\sum_{e=1}^E \frac{\sqrt{n_{min}}}{E} \leq E\sqrt{\frac{\sum_{e=1}^E n_{min}}{E}} \leq \sqrt{TE} \leq \sqrt{TK\log(T)}.$$
Therefore, we achieve a final ``price of optimism'' of
\begin{equation}
    O\left(\gamma N\log(T)\sqrt{TK}\right).
    \label{eq:p_of_opt}
\end{equation}

The second term in (\ref{eq:reg_in_3_terms}) can be bounded using Lemma \ref{lem:numepisodes} achieving a final transition cost of
\begin{equation}
    O(\gamma NDK\log(T)).
    \label{eq:c_of_dest}
\end{equation}

Finally, the third term can be bounded by noting that $t_{e} \geq e + 1$ for each $e$. Then,
\begin{multline}
    \mathbb{E}\left[\sum_{e=1}^E 4\gamma Nt_{e}^{-1} \right]
    \leq 4\gamma N\mathbb{E}\left[\sum_{e=1}^E \frac{1}{e} \right] 
    \\ \leq 4\gamma N\mathbb{E}[\log(E)]
    \leq O(\gamma N [\log(K) + \log\log(T)])
    \label{eq:cost_of_bad}
\end{multline}
where in the second inequality we used the bound on harmonic sums. Therefore, by combining (\ref{eq:p_of_opt}), (\ref{eq:c_of_dest}), and (\ref{eq:cost_of_bad}) and noting the contribution of the initialization phase is dominated, we bound our total regret contribution
\begin{equation}
    \mathbb{E}\left[\sum_{t=1}^T \mathcal{R}_t \right] \leq O\left(\gamma N\log(T)\left[\sqrt{TK} + DK\right]\right).
\end{equation}

\end{document}